\documentclass[12pt]{colt2018}

\numberwithin{equation}{section}

\usepackage[mathcal]{euscript}
\usepackage{verbatim}
\usepackage{cleveref}
\usepackage{dsfont}
\usepackage{enumitem}

\numberwithin{equation}{section}

\usepackage{commath}
\usepackage[normalem]{ulem}

    \def\ddefloop#1{\ifx\ddefloop#1\else\ddef{#1}\expandafter\ddefloop\fi}

    \def\ddef#1{\expandafter\def\csname c#1\endcsname{\ensuremath{\mathcal{#1}}}}
    \ddefloop ABCDEFGHIJKLMNOPQRSTUVWXYZ\ddefloop

    \def\ddef#1{\expandafter\def\csname s#1\endcsname{\ensuremath{\mathsf{#1}}}}
    \ddefloop ABCDEFGHIJKLMNOPQRSTUVWXYZ\ddefloop

    \def\E{\mathbf{E}}
    \def\PP{\mathbf{P}}
	\def\MM{\mathbf{M}}

    \def\Reals{\mathbb{R}}
    \def\Naturals{\mathbb{N}}

    \def\deq{:=}

    \def\bd#1{\mathbf{#1}}
    \def\bz{\bd{z}}
    \def\bZ{\bd{Z}}

    \def\tO{{\tilde{\cO}}}
    
    \def\barw{\overline{w}}

    \def\d{{\mathrm d}}
    \def\1{{\mathbf 1}}

    \def\ave#1{\langle #1 \rangle}
    \def\Ave#1{\left\langle #1 \right\rangle}
    \def\eps{\varepsilon}

\newcommand{\RNum}[1]{\uppercase\expandafter{\romannumeral #1\relax}}

\title[Empirical metastability of Langevin algorithm]{Local Optimality and Generalization Guarantees for the Langevin Algorithm via Empirical Metastability}
\usepackage{times}

 \coltauthor{\Name{Belinda Tzen} \Email{btzen2@illinois.edu}\\
 \addr University of Illinois
 \AND
 \Name{Tengyuan Liang} \Email{tengyuan.liang@chicagobooth.edu}\\
 \addr University of Chicago, Booth School of Business
 \AND
 \Name{Maxim Raginsky} \Email{maxim@illinois.edu}\\
 \addr University of Illinois
 }

\begin{document}

\maketitle

\begin{abstract} We study the detailed path-wise behavior of the discrete-time Langevin algorithm for non-convex Empirical Risk Minimization (ERM) through the lens of metastability, adopting some techniques from \cite{berglund_gentz_pathwise}.
	
	For a particular local optimum of the empirical risk, with an \textit{arbitrary initialization}, we show that, with high probability, at least one of the following two events will occur: (1) the Langevin trajectory ends up somewhere outside the $\eps$-neighborhood of this particular optimum within a short \textit{recurrence time}; (2) it enters this $\eps$-neighborhood by the recurrence time and stays there until a potentially exponentially long \textit{escape time}. We call this phenomenon \textit{empirical metastability}.
	
This two-timescale characterization aligns nicely with the existing literature in the following two senses. First, the effective recurrence time (i.e., number of iterations multiplied by stepsize) is dimension-independent, and resembles the convergence time of continuous-time deterministic Gradient Descent (GD). However unlike GD, the Langevin algorithm does not require strong conditions on local initialization, and has the possibility of eventually visiting all optima. Second, the scaling of the escape time is consistent with the Eyring-Kramers law, which states that the Langevin scheme will eventually visit all local minima, but it will take an exponentially long time to transit among them.
We apply this path-wise concentration result in the context of statistical learning to examine local notions of generalization and optimality. 

\end{abstract}

\section{Introduction and informal summary of results}
While it is a classical algorithm, gradient descent (along with variants, such as SGD) is now one of the most popular tools in large-scale optimization due to its simplicity and speed. 
Consider the following stochastic optimization problem:
\begin{equation*}
\text{minimize} \qquad  F(w) \deq \mathbf{E}_P[f(w,Z)]
\end{equation*}
over $w \in \Reals^d$, where $Z$ is a random element of some space $\sZ$ with (typically unknown) probability law $P$.  We have access to an $n$-tuple $\bZ = (Z_1,\ldots,Z_n)$ of i.i.d.\ samples from $P$, and thus can attempt to minimize the empirical risk
\begin{equation*}
F_{\mathbf{\bZ}}(w) \deq \frac{1}{n}\sum_{i=1}^n f(w,Z_i).
\end{equation*}
For example, we may use vanilla gradient descent, which has the form
\begin{equation*}
W^{(k+1)} = W^{(k)} - \eta\nabla F_{\bZ}(W^{(k)}), \qquad k = 0,1,\ldots
\end{equation*}
where $\eta > 0$ is the step size.

The behavior of gradient descent when the objective function $F$ is convex is well understood.  However, there has recently been much focus on non-convex optimization, motivated by the remarkable success of deep neural nets on problems across numerous disciplines. Gradient descent (or some stochastic variant) is often the algorithm of choice in these settings, but a theoretical characterization of the performance of these methods remains elusive in the non-convex case (cf.\ \cite{natasha,natasha2} and \cite{carmon2017} for some recent developments).
A variant that has proven to be more amenable to analysis is the \textit{Langevin algorithm}, i.e., gradient descent with appropriately scaled isotropic Gaussian noise added in each iteration:  
\begin{equation}\label{eq:discrete_Langevin}
W^{(k+1)} = W^{(k)} - \eta\nabla F_{\bZ}(W^{(k)}) + \sqrt{2\beta^{-1}\eta}\xi^{(k)},
\end{equation}
where $\beta > 0$ is the \textit{inverse temperature}, and the noise sequence $\xi^{(k)} \stackrel{{\rm i.i.d.}}{\sim}\mathcal{N}(0,I_d)$ is independent of the initial point $W^{(0)}$. This modest modification  preserves many of the desirable properties of gradient descent, while ensuring that it can aptly navigate a landscape containing multiple critical points.  The analysis of \eqref{eq:discrete_Langevin} is facilitated by the fact that it can be viewed as a discrete-time approximation to a continuous-time Langevin diffusion, described by the It\^o stochastic differential equation
\begin{equation}\label{eq:Langevin}
\mathrm{d}W_t = -\nabla F_{\bZ}(W_t)\mathrm{d}t + \sqrt{2\beta^{-1}}\mathrm{d}B_t
\end{equation}
where $(B_t)_{t\geq 0}$ is a standard $d$-dimensional Brownian motion (e.g., \cite{borkar1999strong}). 

The fact that the injection of Gaussian noise will cause the Langevin algorithm to asymptotically converge to a global minimum has long been known from the physics literature, and revisited more recently by, e.g., \cite{wellingteh2011sgld}. In the context of empirical risk minimization, the first non-asymptotic guarantee of global convergence of the Langevin algorithm was given by \cite*{rrt_colt17}, who showed that the Langevin algorithm with stochastic gradients  converges to $\eps$-approximate global minimizers after $\mathrm{poly}\big(\frac{1}{\eps},\beta,d,\frac{1}{\lambda^*}\big)$ iterations, where $\lambda^*$ is a spectral gap parameter that reflects the rate of convergence of the Langevin diffusion process to its stationary distribution, and is exponential in both $d$ and $\beta$ in general.  A complementary result of \cite*{zhang2017hitting} shows that a certain modified version of the stochastic gradient Langevin algorithm will hit an $\eps$-approximate local minimum in time polynomial in all parameters.

The Gaussian perturbation inherent in the Langevin scheme is certainly necessary for ensuring its eventual convergence to a global optimum. However, even the vanilla gradient method can hit a local optimum (or at least an approximate stationary point) in polynomial time. To provide a more complete characterization of the behavior of the Langevin algorithm on non-convex problems, we analyze it over three timescales:  colloquially, short, intermediate, and long.  We show that, with high probability, the iterates of \eqref{eq:discrete_Langevin} will either be at least $\eps$ away from some local minimum within the \textit{recurrence time} $K_0  = \tO\big(\frac{1}{\eta}\log \frac{1}{\eps}\big)$, or fall within an $\eps$-neighborhood of some local minimum and remain there through the \textit{escape time} $K  =\tO\big(\frac{1}{\eta}\big(T + \log \frac{1}{\eps}\big)\big)$, where  $T = {\rm poly}(d)$ characterizes the intermediate timescales, while $T = {\rm exp}(d)$ describes the long-time behavior. We refer to this behavior as \textit{empirical metastability}. Moreover, we build on this result to provide generalization bounds under mild assumptions on the population risk function, e.g., those in recent work of \cite*{mei2016landscape} that demonstrates a close correspondence between the landscapes of the population and empirical risk.  Such a nuanced characterization of the Langevin algorithm is valuable for machine learning because it corroborates the empirical evidence that minimizers are approached quickly; provides good generalization guarantees; and ensures that a local minimum can be escaped and others reached in a non-asymptotic setting.

\subsection{Method of analysis:  an overview}

Our analysis draws upon ideas and techniques from metastability theory, which is concerned with ``the long-term behavior of dynamical systems under the influence of weak random perturbations'' (\cite{bovier2016metastability}),  where the system spends a large but random amount of time within a given region of its state space and subsequently transits to another due to noise.  Metastability theory offers a view of the behavior of the Langevin algorithm over the risk landscape that is congruent with the interpretation of it as the discretization of the Langevin diffusion process. 

As in \cite*{rrt_colt17}, the analysis proceeds by first characterizing the behavior of the continuous-time Langevin diffusion \eqref{eq:Langevin}. In that context, the fast recurrence and slow escape phenomena can be stated precisely as follows:
Starting at an arbitrary distance $r$ from some local minimum of the empirical risk, with probability at least $1-\delta$, the diffusion with sufficiently large $\beta$ will either end up at distance at least $\eps+re^{-t}$ from the minimum by the recurrence time $T_{\rm rec} = \tO\big(\log \frac{r}{\eps}\big)$; or will approach it within distance $\eps + re^{-t}$ at each  $t$ through the scape time $T_{\rm esc} = \tO\big(e^d + \log \frac{r}{\eps}\big)$. The exponential scaling of the eventual escape time is consistent with the Eyring-Kramers law, which guarantees exponentially long mean transition times between local optima \citep{bovier2004metastability,Olivieri_Vares_metastability,bovier2016metastability}.

In order to prove this, we adopt the path-wise concentration approach of \cite{berglund_gentz_pathwise}.  We linearize the empirical gradient $\nabla F_\bZ$ around some local minimum $\barw_\bZ$. This allows us to decompose the diffusion process, after an appropriate transformation, into a stable Gaussian process and a remainder term.  The latter can be controlled using a delicate stopping-time analysis.  The former lends itself to the construction of a martingale, whose behavior can be precisely described over time intervals of desired length.  In particular, over small time intervals, we can obtain sharp Gaussian tail bounds on the the  maximal deviation of that martingale, and extend the result to larger time intervals via the union bound. Using this decomposition, we establish the fast recurrence and slow escape result for the Langevin diffusion.

Next, we relate the continuous-time diffusion to the discrete-time Langevin algorithm via the usual linear interpolation and the Girsanov theorem, as in \cite*{rrt_colt17}.  The Langevin algorithm iterates have the same joint distribution as that of the interpolation at the corresponding points, which allows us to construct a coupling of the discrete-time Langevin algorithm \eqref{eq:discrete_Langevin} and the continuous-time diffusion \eqref{eq:Langevin} sampled at integer multiples of $\eta$, such that $W^{(k)}=W_{k\eta}$ for all $k \le K$ with high probability. However, in order to transfer the continuous-time metastability result to the discrete-time setting, we also need to take care to account for the behavior of the diffusion in the intervening intervals $(k \eta, (k+1) \eta)$. This is done via a quantitative continuity estimate for the diffusion using Gronwall's lemma and the reflection principle for the Brownian motion \citep{morters2010brownian}. 

Finally, we use the metastability results to provide local optimality and generalization guarantees for the Langevin algorithm. The main quantity of interest is the difference 
$$
F(\barw_\bZ) - \min_{K_0 \le k \le K} F_\bZ(W^{(k)}),
$$
where $K_0$ is the recurrence time of the Langevin algorithm. Note that, while $F(\barw_\bZ)$ involves the population risk, the term $\min_{K_0 \le k \le K} F_\bZ(W^{(k)})$ can be readily computed from the trajectory of the Langevin algorithm. To control this difference, we use the techniques recently developed by \cite*{mei2016landscape} in order to relate the local geometry of the empirical risk to that of the population risk. The main message  is that, with high probability, there exists a point $w \in \Reals^d$,  whose population risk is upper-bounded by the minimal empirical risk along the trajectory of the Langevin algorithm plus a term that scales like $\sqrt{(d/n)\log n}$.

\subsection{Notation}

\sloppypar Any positive definite matrix $A \in \Reals^{d \times d}$ induces a norm $\| w \|_A \deq \sqrt{\ave{w,Aw}}$. We denote by $\sB^d_A(v,r)$ the ball of radius $r$ with center $v \in \Reals^d$ in this norm: $\sB^d_A(v,r) \deq \{ w \in \Reals^d : \| w - v \|_A \le r \}$. The usual Euclidean norm on $\Reals^d$ is denoted by $\|\cdot\|$. We denote by $\sB^d(v,r)$ the Euclidean ball of radius $r$ centered at $v$; when $v=0$, we simply write $\sB^d(r)$. We denote by $\lambda_1(A),\ldots,\lambda_d(A)$ the eigenvalues of a symmetric matrix $A \in \Reals^{d\times d}$, and by $\|A\|$ the spectral norm of $A$.

\section{Main Results}

We impose the following assumptions:
\begin{enumerate}[label={\bf (A.\arabic*)}]
	\item The functions $f(\cdot,z)$ are $C^2$, and there exist constants $A,B,C \ge 0$, such that
	\begin{align*}
		|f(0,z)| \le A, \qquad \| \nabla f(0,z) \| \le B, \qquad \| \nabla^2 f(0,z) \| \le C, \qquad \forall z \in \sZ.
		\end{align*}
    \item The functions $f(\cdot,z)$  have Lipschitz-continuous gradients and Hessians, uniformly in $z \in \sZ$: there exist constants $L,M > 0$, such that, for all $w,v \in \Reals^d$,
	\begin{align*}
	\|\nabla f(w,z)-\nabla f(v,z)\| \le M \| w - v\| \quad \text{and} \quad
	\|\nabla^2 f(w,z) - \nabla^2 f(v,z)\| \le L \|w - v \|.
\end{align*}
    \item For each $\bz \in \sZ^n$, the empirical risk $F_\bz(\cdot)$ is \textit{$(m,b)$-dissipative}: for some $m > 0$ and $b \ge 0$,
	\begin{align*}
		\ave{w, \nabla F_\bz(w)} \ge m \| w \|^2 - b, \qquad \forall w \in \Reals^d.
	\end{align*}
\end{enumerate}
The first two assumptions are standard. For the discussion of the dissipativity assumption, see \cite*{rrt_colt17}. It is not hard to see that, if it holds, then any critical point (i.e., the one where the gradient is equal to zero) of both the empirical risk $F_\bZ$ and the population risk $F$ is contained in the ball $\sB^d(R)$ with $R = \sqrt{b/m}$.

\subsection{Empirical metastability}

Our first result concerns the behavior of the Langevin algorithm \eqref{eq:discrete_Langevin} with respect to an arbitrary nondegenerate local minimum $\barw_\bZ$ of the empirical risk $F_\bZ$ -- that is, the Hessian $H \deq \nabla^2 F_\bZ(\barw_\bZ)$ is positive definite. Without loss of generality, we may assume that $\min_{j \le d} \lambda_j(H) \ge m$.

		\begin{theorem}\label{thm:emp_metastable} Fix some $\delta \in (0,1)$ and $r > 0$. There exist absolute constants $c_1,c_2 > 0$, such that the following holds. Assume $\eps \in (0,\frac{c_1m^2}{L\sqrt{M}} \wedge 8r)$ and define the {\em recurrence time} $T_{\rm rec} \deq \frac{2}{m}\log \frac{8r}{\eps}$ and the {\em escape time} $T_{\rm esc} \deq T_{\rm rec} + T$ for an arbitrary $T > 0$. Consider the Langevin algorithm \eqref{eq:discrete_Langevin} with initial point $w \in \sB^d_H(\barw_\bZ,r) \cap \sB^d(R)$, where the step size $\eta$ and the inverse temperature $\beta$ satisfy the following conditions:
\begin{align*}
\eta \le 1 \wedge \dfrac{m}{2M^2} \wedge \dfrac{c_1\delta^2}{M^2(\beta G_0+d)T_{\rm esc}} \wedge \dfrac{c_1 \delta\eps^2}{M^3 G_1 T_{\rm rec}} \text{ and } \beta \ge \dfrac{c_2}{\eps^2}\left(d + \log \dfrac{MT_{\rm esc}}{\delta}\right) \vee \dfrac{c_2Md}{\eps^2}\log \dfrac{dT_{\rm rec}}{\delta \eta},
\end{align*}
where 
			 \begin{align}\label{eq:grad_bound}
				 G_0 := 2M^2 \left( R^2 + 2\left(1 \vee \dfrac{1}{m}\right) \left(b+B^2+\dfrac{d}{\beta}\right) \right) + 2B^2, \quad G_1 \deq R + \frac{b+d}{m}.
			\end{align}
			Then, for any realization of $\bZ$, with probability at least $1-\delta$ with respect to the Gaussian noise sequence $\xi^{(1)},\xi^{(2)},\ldots$, at least one of the following two events will occur:
			\begin{enumerate}
				\item $\|W^{(k)} - \barw_\bZ\|_H \ge \frac{1}{2}\left(\eps + re^{-mk\eta}\right)$ for some $k \le \eta^{-1}T_{\rm rec}$
				\item $\|W^{(k)} - \barw_\bZ\|_H \le \eps + re^{-mk\eta}$ for all iterations $\eta^{-1}T_{\rm rec} \le k \le \eta^{-1}T_{\rm esc}$.
			\end{enumerate}
    \end{theorem}
	
\begin{remark} {\em It is not hard to show that $K = \eta^{-1}T_{\rm esc}$ is polynomial in $d$, $T$, $1/\eps$, $1/\delta$. By taking $r = \eps/8$, so that $T_{\rm rec} = 0$, and $T = \exp(d)$, we recover the behavior consistent with the Eyring--Kramers law: If the Langevin scheme \eqref{eq:discrete_Langevin} is initialized in an $\eps$-neighborhood of some local minimum, it will remain there for an exponentially long time with high probability before transiting to the basin of attraction of some other local minimum.}
\end{remark}

\subsection{Local optimality and generalization}

\begin{theorem}\label{thm:margin-generalization} Suppose that $f$ satisfies Assumptions~{\bf(A.1)}--{\bf(A.3)}, and that the population risk $F$ is {\em $(2\eps_0,2m)$-strongly Morse} in the sense of \cite*{mei2016landscape}:
	\begin{align*}
		\| \nabla F(w) \| \leq 2\eps_0 ~~\text{implies}~~ \min_{j \in [d]} |\lambda_j (\nabla^2 F(w))| \geq 2m.
	\end{align*}
Consider running the discrete-time Langevin algorithm \eqref{eq:discrete_Langevin} with $W^{(0)}=w$, where $\beta$ and $\eta$ are chosen as in Theorem~\ref{thm:emp_metastable}. Suppose also that the sample size satisfies $n \ge cd\log d$ and $\dfrac{n}{\log n} \ge \dfrac{c\sigma^2d}{(\eps_0 \wedge m)^2}$, where $c = c_0 \left(1 \vee \log((M\vee L\vee (B+MR))R\sigma/\delta)\right)$ for some absolute constant $c_0$ and $\sigma  \deq (A+(B+MR)R) \vee (B + MR) \vee (C + LR)$, and that $\eps \le \frac{3m^{3/2}}{2L}$. Then with probability at least $1-\delta$ (with respect to the training data $\bZ$ and the isotropic Gaussian noise added by the Langevin algorithm),  for any local minimum $\barw_\bZ$ of the empirical risk $F_\bZ$, either $\|W^{(k)}-\barw_\bZ\|_H \ge 2\eps$ for some $k \le \eta^{-1}T_{\rm rec}$ or
\begin{align}
	F(\barw_\bZ) \le \min_{\eta^{-1}T_{\rm rec} \le k \le \eta^{-1}T_{\rm esc}} F_\bZ(W^{(k)}) + \sigma\sqrt{\frac{cd\log n}{n}},
\end{align}
where $r = \|w-\barw_\bZ\|_H$ and $H = \nabla^2F(\barw_\bZ)$.
\end{theorem}

\section{Proof of Theorem~\ref{thm:emp_metastable}}

\subsection{Preliminaries: a martingale concentration lemma}

We first define some auxiliary quantities and establish a useful result for the Langevin diffusion \eqref{eq:Langevin}, along the lines of \cite{berglund_gentz_pathwise}.

We begin by linearizing the gradient $\nabla F_\bZ$ around $\barw_\bZ$: Recalling that $H = \nabla^2 F_\bZ(\barw_\bZ)$ and writing $W_t = Y_t + \barw_\bZ$, we have $\nabla F_\bZ(W_t) = HY_t - \rho(Y_t)$, where the remainder satisfies $\|\rho(Y_t)\| \le \frac{L}{2}\|Y_t\|^2$ since the Hessian of $F_\bZ$ is $L$-Lipschitz \citep[Lemma~1.2.4]{nesterov2013introductory}.  The diffusion process $\{Y_t\}^\infty_{t \ge 0}$ obeys the It\^{o} SDE
    \begin{align}\label{eq:displacement_SDE}
    	\d Y_t = -(HY_t - \rho(Y_t))\d t + \sqrt{2\beta^{-1}}\d B_t
    \end{align}
    with the initial condition $Y_0 = w - \barw_\bZ$. The solution of \eqref{eq:displacement_SDE} is given by
    \begin{align*}
    	Y_t = e^{-tH}Y_0 + \sqrt{2\beta^{-1}}\int^t_0 e^{(s-t)H}\d B_s + \int^t_0 e^{(s-t)H}\rho(Y_s) \d s,
    \end{align*}
	where $e^{A} \deq \sum^\infty_{k=0} \frac{A^k}{k!}$, for $A \in \Reals^{d \times d}$, is the matrix exponential. 

Given  $0 \le t_0 \le t_1$, let us define the matrix flow $Q_{t_0}(t) \deq H^{1/2}e^{(t_0-t)H}$ and the It\^{o} process
		\begin{align*}
			Z_t &\deq e^{(t-t_0)H}Y_t \\
			&= e^{-t_0H}Y_0 + \sqrt{2\beta^{-1}}\int^t_0 e^{(s-t_0)H}\d B_s + \int^t_0 e^{(s-t_0)H}\rho(Y_s)\d s,
		\end{align*}
	for $t \in [t_0,t_1]$. These definitions are motivated by the observation that $\| Y_t \|_H = \| H^{1/2}Y_t \| = \|Q_{t_0}(t)Z_t \|$. Next, we decompose $Z_t$ into the Ornstein--Uhlenbeck term
	    \begin{align*}
	    	Z^0_t \deq e^{-t_0H}Y_0 + \sqrt{2\beta^{-1}}\int^t_0 e^{(s-t_0)H} \d B_s
	    \end{align*}
	    and the remainder term
	    \begin{align*}
	    	Z^1_t \deq \int^t_0 e^{(s-t_0)H}\rho(Y_s)\d s.
	    \end{align*}
Note that, for any $t \in [t_0,t_1]$,
\begin{align*}
	Q_{t_0}(t_1)Z^0_t = H^{1/2}e^{-t_1H}Y_0 + \sqrt{2\beta^{-1}}\int^t_0 H^{1/2}e^{(s-t_1)H}\d B_s,
\end{align*}
so the process $\{ Q_{t_0}(t_1)Z^0_t\}_{t \in [t_0,t_1]}$ is a martingale. This leads to the following tail bound (see Appendix~\ref{app:mart_tail} for the proof):

	\begin{lemma}\label{lm:mart_tail} For any $\lambda \in (0,1/2)$, $h > 0$, and $y_0 \in \Reals^d$,
		\begin{align}\label{eq:Z0_tail_bound}
			\PP^{y_0} \left[\sup_{t_0 \le t \le t_1} \|Q_{t_0}(t_1)Z^0_t \| \ge h \right] \le \left(\frac{1}{1-\lambda}\right)^{d/2}\exp\left(-\frac{\beta\lambda}{2}\left[h^2-\ave{\mu,(I-\beta\lambda \Sigma)^{-1}\mu}\right]\right),
		\end{align}
		where
	$\mu \deq H^{1/2}e^{-t_1H}y_0$ and $\Sigma \deq \beta^{-1}(I-e^{-2t_1H})$, and $\PP^{y_0}[\cdot] \deq \PP[\cdot|Y_0=y_0]$.
	\end{lemma}

\subsection{Langevin diffusion: fast recurrence and slow escape}

We now provide the quantitative path-wise metastability result for the Langevin diffusion.

\begin{proposition}\label{prop:FR_SE} Fix any $r > 0$ and $\eps \in (0, \frac{(\sqrt{2}-1)m^2}{4L\sqrt{2M}}\wedge 8r)$, and consider
the stopping time
\begin{align*}
	\tau \deq \inf\left\{ t \ge 0: \frac{\|W_t - \barw_\bZ\|_H}{\eps + r e^{-mt}} \ge 1 \right\}.
\end{align*}
Then, for any initial point $w \in \sB^d_H(\barw_\bZ,r)$,
\begin{align*}
	\PP^{w}\left[ \tau \in [T_{\rm rec},T_{\rm esc}]\right] \le \delta,
\end{align*}
provided $\beta \ge \frac{128}{3\eps^2}\left(d +\log \frac{2MT+1}{\delta}\right)$, where $\PP^w[\cdot] \deq \PP[\cdot|W_0 = w]$.
\end{proposition}

\begin{proof}
Since $\|Y_0\|_H \le r$, we know that $\tau > 0$. Fix some $T_{\rm rec} \le t_0 \le t_1$, such that $t_1 - t_0 \le \frac{1}{2M}$. For every $t \in [t_0,t_1]$, $Y_t = W_t - \barw_\bZ$ satisfies
\begin{align*}
	\|Y_t\|_H = \|Q_{t_0}(t)Z_t\| = \|e^{(t_1-t)H}Q_{t_0}(t_1)Z_t\| \le e^{1/2}\|Q_{t_0}(t_1)Z_t\|.
\end{align*}
Therefore,
\begin{align*}
	&\PP^w\left[\tau \in [t_0,t_1]\right] \nonumber\\
	&= \PP^{y_0}\left[ \sup_{t_0 \le t \le t_1 \wedge \tau}\frac{\|Y_t\|_H}{\eps + r e^{-mt}} \ge 1, \, \tau \ge t_0 \right] \\
	&\le \PP^{y_0}\left[ \sup_{t_0 \le t \le t_1 \wedge \tau}\frac{\|Q_{t_0}(t_1)Z_t\|}{\eps + r e^{-mt}} \ge \frac{1}{2}, \, \tau \ge t_0 \right] \\
	&\le \PP^{y_0}\left[ \sup_{t_0 \le t \le t_1 \wedge \tau}\frac{\|Q_{t_0}(t_1)Z^0_t\|}{\eps + r e^{-mt}} \ge c_0, \, \tau \ge t_0 \right] + \PP^{y_0}\left[ \sup_{t_0 \le t \le t_1 \wedge \tau}\frac{\|Q_{t_0}(t_1)Z^1_t\|}{\eps + r e^{-mt}} \ge c_1, \, \tau \ge t_0 \right] \\
	&=: P_0 + P_1,
\end{align*}
for any choice of $c_0,c_1 > 0$ satisfying $c_0 + c_1 = \frac{1}{2}$.

	We first upper-bound $P_1$.  On the event $\tau \in [t_0, t_1]$, for any $0\leq s \le t_1 \wedge \tau$, $\| \rho(Y_s) \| \leq \frac{L}{2}\| Y_s\|^2 \leq \frac{L}{2m} (\eps + re^{-m s})^2$. Therefore, for any $t \in [t_0,t_1 \wedge \tau]$,
	    \begin{align*}
	    	\|Q_{t_0}(t_1)Z^1_t\| &\le \|H^{1/2}\|\int^t_0 \|e^{(s-t_1)H}\rho(Y_s)\|\d s \\
			&\le \|H^{1/2}\| \int^t_0 \|e^{(s-t_1)H}\| \|\rho(Y_s)\| \d s \\
			&\le \frac{L\sqrt{M}}{2m} \int^t_0 e^{(s-t_1)m} (\eps + re^{-m s})^2 \d s \\
			&\le \frac{L\sqrt{M}}{m^2} \cdot(\eps^2 + r^2 e^{-mt}) \\
			&< \frac{2L\sqrt{M}}{m^2}\eps^2
\end{align*}
since $t \ge t_0 \ge T_{\rm rec}$. Consequently, if we take $c_1 = \frac{2L\sqrt{M}}{m^2}\eps$, then
\begin{align*}
	\sup_{t_0 \le t \le t_1 \wedge \tau} \frac{\|Q_{t_0}(t_1)Z_t\|}{\eps + re^{-mt}} \le \frac{1}{\eps}\sup_{t_0 \le t \le t_1 \wedge \tau} \|Q_{t_0}(t_1)Z_t\| < c_1,
\end{align*}
which implies $P_1 = 0$. Moreover, $c_0 = \frac{1}{2}-c_1 = \frac{1}{2} - \frac{2L\sqrt{M}}{m^2}\eps > \frac{1}{4}$.

We next estimate $P_0$. Since $\|y_0\|^2_H \le r^2$,
	\begin{align*}
	& \Ave{H^{1/2}y_0,e^{-t_1H}\left(I - \frac{3}{4}(I-e^{-2t_1H})\right)^{-1}e^{-t_1H}H^{1/2}y_0} \nonumber\\
	&\qquad \le \left\| e^{-2t_1H} \left(I - \frac{3}{4}(I-e^{-2t_1H})\right)^{-1} \right\| \|y_0 \|^2_H \\
	&\qquad \le  4e^{-2m t_1} r^2 \le\eps^2/16.
	\end{align*}
Therefore, by Lemma~\ref{lm:mart_tail} with $h= (\eps + r e^{-mt_1})c_0$ and $\lambda = \frac{3}{4}$,
	\begin{align*}
		P_0 &  \leq \PP^{y_0}\left[\sup_{t_0 \le t \le t_1} \frac{\|Q_{t_0}(t_1)Z_t^0 \|}{\eps + r e^{-mt}} \ge c_0 \right]  \\
		& \leq \PP^{y_0}\left[\sup_{t_0 \le t \le t_1} \|Q_{t_0}(t_1)Z_t^0 \| \ge (\eps + r e^{-mt_1})c_0 \right] \\
		& \leq 2^d \exp\left( -\frac{3\beta\eps^2}{8}\left(c_0^2 - \frac{1}{16}\right)\right) \\
		& \le 2^d \exp\left(-\frac{3\beta\eps^2}{128}\right).
	\end{align*}
Thus, for any $t_0 \ge T_{\rm rec}$ and $t_0 \le t_1 \le t_0 + \frac{1}{2M}$,
\begin{align}\label{eq:nohit_small_timescale}
	\PP^{w}\left[ \tau \in [t_0,t_1]\right] \le 2^d \exp\left(-\frac{3\beta\eps^2}{128}\right).
\end{align}
Now fix an arbitrary $T > 0$ and recall the definition of the escape time $T_{\rm esc} = T + T_{\rm rec}$. Let $J:=\lceil 2MT \rceil$ and partition the interval $[0,T_{\rm esc}]$ using the points $T_{\rm rec} = t_0 < t_1 < \ldots < t_J = T_{\rm esc}$ with $t_j = \frac{j}{2M}$ for $j = 0, 1,\ldots, J-1$. Then, using the union bound and Eq.~\eqref{eq:nohit_small_timescale}, we obtain
	\begin{align*}
		\PP^{w}\left[ \tau \in [T_{\rm rec}, T_{\rm esc}]  \right] &= \sum^{J-1}_{j=0} \PP^{w}\left[\tau \in [t_j,t_{j+1}]\right] \\
		& \leq (2M T +1) 2^d \exp\left(-\frac{3\beta\eps^2}{128}\right),
	\end{align*}
and this probability will be smaller than $\delta$ if we choose $\beta$ accordingly.
\end{proof}

\subsection{A path coupling argument}

Next, we relate the discrete Langevin algorithm \eqref{eq:discrete_Langevin} to the diffusion \eqref{eq:Langevin} sampled at $t=0,\eta,2\eta,\ldots$. To that end, we introduce the standard continuous-time interpolation 
\begin{align*}
	V_t \deq V_0 - \int^t_0 \nabla F_\bZ(V_{\lfloor s/\eta \rfloor\eta}) \d s + \sqrt{2\beta^{-1}}\int^t_0 \d B_s,
\end{align*}
with the deterministic initial condition $V_0 = W^{(0)} = w$. Note that, for any $K$, the probability law $\mu_K$ of the random vector $(W^{(1)},\ldots,W^{(K)})$ is equal to the probability law $\nu_K$ of the random vector $(V_\eta,\ldots,V_{K\eta})$. 

Let us denote by $\PP^t_V$ and $\PP^t_W$ the probability laws of $(V_s: 0 \le s \le t)$ and $(W_s: 0 \le s \le t)$, respectively, with the same deterministic initialization $V_0 = W_0 = w$. Then we can invoke Girsanov's theorem in the same way as in \cite{dalalyan2017theoretical} and \cite*{rrt_colt17} to show that, for any $K \in \Naturals$,
\begin{align}\label{eq:KL_bound_1}
	D\big(\PP^{K\eta}_V \big\| \PP^{K\eta}_W \big) \le \frac{\beta M^2\eta^3}{2}\sum^{K-1}_{k=0} \E^{w} \|\nabla F_\bZ(W^{(k)})\|^2 + KdM^2\eta^2,
\end{align}
where $D(\cdot\|\cdot)$ denotes the relative entropy. The expectations of the squared norms of the gradients can be upper-bounded using Lemmas~2 and 3 of \cite*{rrt_colt17}: under {\bf (A.1)}--{\bf (A.3)}, $\sup_{k \ge 0} \E^w \| \nabla F_\bZ(W^{(k)})\|^2 \le  G_0$ for all $\eta < 1 \wedge \frac{m}{M^2}$, where $G_0$ is defined in \eqref{eq:grad_bound}. Therefore, 
\begin{align*}
	D(\mu_K \| \nu_K) \le D\big(\PP^{K\eta}_V \big\| \PP^{K\eta}_W \big) \le M^2\left(\frac{\beta G_0}{2}+d\right)K\eta^2,
\end{align*}
where the first step is by the data processing inequality for the relative entropy, and the second one follows upon substituting the gradient bound into \eqref{eq:KL_bound_1}. Using Pinsker's inequality, we obtain
\begin{align*}
	\|\mu_K - \nu_K\|^2_{\rm TV} \le  \frac{M^2}{2}\left(\frac{\beta G_0}{2}+d\right)K\eta^2.
\end{align*}
Now we recall the following result about an optimal coupling \citep[Theorem~5.2]{Lindvall_coupling}: Given any two random elements $X,Y$ of a common standard Borel space $\sX$, there exists a coupling $\MM$ of $X$ and $Y$, i.e., a probability measure $\MM$ on the product space $\sX \times \sX$, such that $\MM(\cdot \times \sX) = \cL(X)$, $\MM(\sX \times \cdot) = \cL(Y)$, and 
\begin{align*}
	\MM[X \neq Y] \le \|\cL(X)-\cL(Y)\|_{\rm TV}
\end{align*}
Hence, given any $\beta > 0$ and any $K \in \Naturals$ satisfying $K\eta \le T_{\rm esc}$, we can choose $\eta \le \frac{4\delta^2}{M^2(\beta G_0 + 2d)T_{\rm esc}}$ to ensure that there exists a coupling $\MM$ of $(W^{(k)} : k \in [K])$ and $(W_{k\eta} : k \in [K])$, such that
\begin{align*}
	\MM\big((W^{(1)},W^{(2)},\ldots,W^{(K)}) \neq (W_\eta,W_{2\eta},\ldots,W_{K\eta})\big) \le \delta.
\end{align*}
In that case, we have 
\begin{align}\label{eq:path_coupling}
&	\PP\big((W^{(1)},\ldots,W^{(K)}) \in \cdot \big) \nonumber\\
& \qquad = \MM\big((W^{(1)},\ldots,W^{(K)}) \in \cdot \big)  \nonumber\\
& \qquad \le \MM\big((W_\eta,\ldots,W_{K\eta}) \in \cdot \big) + \MM\big( (W^{(1)},\ldots,W^{(K)}) \neq (W_1,\ldots,W_{K\eta})\big) \nonumber\\
& \qquad \le \PP\big((W_\eta,\ldots,W_{K\eta}) \in \cdot  \big)  + \delta.
\end{align}

\subsection{Completing the proof}

Consider the discrete-time Langevin algorithm \eqref{eq:discrete_Langevin}. We need to upper-bound the probability of the event $A = A_1 \cap A_2$, where, for $K = \lfloor\eta^{-1}T_{\rm esc}\rfloor$,
\begin{align*}
A_1 &\deq \left\{ (w^{(1)},\ldots,w^{(K)}) \in \Reals^d \times \ldots \times \Reals^d : \max_{k \le \eta^{-1}T_{\rm rec}} \frac{\|w^{(k)}-\barw_\bZ\|_H}{\eps + r e^{-mk\eta}}\le \frac{1}{2}\right\}, \\
A_2 &\deq \left\{ (w^{(1)},\ldots,w^{(K)}) \in \Reals^d \times \ldots \times \Reals^d : \max_{\eta^{-1}T_{\rm rec} \le k \le K} \frac{\|w^{(k)}-\barw_\bZ\|_H}{\eps + r e^{-mk\eta}} \ge 1 \right\}.
\end{align*}
By Proposition~\ref{prop:FR_SE}, we can choose $\beta > 0$ large enough, so that, with probability at least $1-\delta/3$, for the continuous-time Langevin diffusion \eqref{eq:Langevin} we have either $\|W_t - \barw_\bZ\|_H \ge \eps + re^{-mt}$ for some $t \le T_{\rm esc}$ or $\|W_t - \barw_\bZ\|_H \le \eps + re^{-mt}$ for all $t  \le T_{\rm esc}$. Moreover, for any $K$, $\eta$, and $\beta$ satisfying the conditions of the theorem, there exists a coupling $\MM$ of $(W^{(1)},\ldots,W^{(K)})$ and $(W_\eta,\ldots,W_{K\eta})$, such that, with probability $1-\delta/3$, $W^{(k)}=W_{k\eta}$ for all $k \in [K]$. Then,
using the path coupling estimate \eqref{eq:path_coupling}, we can write
\begin{align*}
	\PP^w[(W^{(1)},\ldots,W^{(K)}) \in A] \le \PP^w[(W_\eta,\ldots,W_{K\eta}) \in A] + \frac{\delta}{3}.
\end{align*}
It remains to estimate the probability of ${A}_1 \cap {A}_2$ for the Langevin diffusion \eqref{eq:Langevin}. For $K_0 \deq \lceil T_{\rm rec}\eta^{-1}\rceil$, partition the interval $[0,T_{\rm rec}]$ using the points $0 = t_0 < t_1 < \ldots < t_{K_0} = T_{\rm rec}$ with $t_k = k\eta$ for $k = 0,1,\ldots,K_0-1$, and consider the event
\begin{align*}
	B \deq \left\{ \max_{0 \le k \le K_0-1} \max_{t \in [t_k,t_{k+1}]}  \|W_t - W_{t_{k+1}}\| \le \frac{\eps}{2\sqrt{M}}\right\}.
\end{align*}
On the event $(W_\eta,\ldots,W_{K\eta}) \in A_1 \cap B$,
\begin{align*}
	\sup_{t \in [0,T_{\rm rec}]} \frac{\|W_t - \barw_\bZ\|_H}{\eps + re^{-ms}} &= \max_{0 \le k \le K_0-1} \sup_{t \in [t_k,t_{k+1}]} \frac{\|W_t - \barw_\bZ\|_H}{\eps + re^{-mt}} \\
	&\le \frac{1}{2} + \max_{0 \le k \le K_0-1} \max_{t \in [t_k,t_{k+1}]} \frac{\sqrt{M}}{\eps} \|W_t - W_{t_{k+1}}\| \le 1,
\end{align*}
hence
\begin{align*}
	\PP^w[(W_\eta,\ldots,W_{K\eta}) \in A] &\le \PP^w[(W_\eta,\ldots,W_{K\eta}) \in A \cap B] + \PP^w[(W_\eta,\ldots,W_{K\eta}) \in B^c] \\
	&\le \PP^w\left[ \tau \in [T_{\rm rec}, T_{\rm esc}]\right] + \PP^w[(W_\eta,\ldots,W_{K\eta}) \in B^c] \\
	&\le \frac{\delta}{3} + \PP^w[(W_\eta,\ldots,W_{K\eta}) \in B^c].
\end{align*}
To complete the proof, we need to arrange  $\PP^w[(W_\eta,\ldots,W_{K\eta}) \in B^c] \le \frac{\delta}{3}$. For $t \in [t_k,t_{k+1}]$,
\begin{align*}
	\|W_t - W_{t_{k+1}}\| &\le \int^{t_{k+1}}_t \| \nabla F_\bZ(W_s)\| \d s + \sqrt{2\beta^{-1}} \|B_t - B_{t_{k+1}}\| \\
	&\le M\int^{t_{k+1}}_t \|W_s - W_{t_{k+1}}\|\d s + M\eta\|W_{t_{k+1}}\| + \sqrt{2\beta^{-1}}\|B_t - B_{t_{k+1}}\|.
\end{align*}
Therefore, using Gronwall's lemma, we obtain
\begin{align*}
	\sup_{t \in [t_k,t_{k+1}]} \|W_t - W_{t_{k+1}}\| \le e^{M\eta}\left[M\eta \|W_{t_{k+1}}\| + \sqrt{2\beta^{-1}}\sup_{t \in [t_k,t_{k+1}]}\|B_t - B_{t_{k+1}}\|\right].
\end{align*}
By Lemma~3 in \cite*{rrt_colt17} and Markov's inequality, for any $u > 0$
\begin{align}\label{eq:Markov_term}
	\PP^w\left[\|W_{t_{k+1}}\| \ge u \right] \le \frac{\sup_{t \ge 0}\E^w\|W_t\|^2}{u^2} \le \frac{G_1}{u^2},
\end{align}
with $G_1$ in \eqref{eq:grad_bound}. By the reflection principle for the Brownian motion \citep{morters2010brownian},
\begin{align}
	\PP\left[ \sup_{t \in [t_k,t_{k+1}]} \|B_t - B_{t_{k+1}}\| \ge u \right] \le 2de^{-cu^2/d\eta}
\end{align}
for some absolute constant $c > 0$. Therefore,
\begin{align}
&	\PP^w[(W_\eta,\ldots,W_{K\eta}) \in B^c] \nonumber\\
&\le \sum^{K_0-1}_{k=0} \left(\PP^w\left[ \|W_{t_{k+1}}\| \ge \frac{\eps e^{-M\eta}}{4M^{3/2}\eta} \right] + \PP\left[\sup_{t \in [t_k,t_{k+1}]} \|B_t - B_{t_{k+1}}\| \ge \eps e^{-M\eta}\sqrt{\frac{\beta}{32M}}\right]\right) \nonumber\\
	&\le K_0 \left( \frac{16G_1M^3\eta^2}{\eps^2} e^{2M\eta} + 2d \exp\left(-\frac{c'\beta\eps^2}{Md\eta}e^{-2M\eta}\right) \right)\label{eq:PBc}
\end{align}
for some absolute constant $c' > 0$. We can first choose $\eta > 0$ small enough to make the first term in \eqref{eq:PBc} smaller than $\delta/6$, and then $\beta > 0$ large enough to make the second term smaller than $\delta/6$.

\section{Proof of Theorem~\ref{thm:margin-generalization}}

\subsection{Preliminaries}

The following lemma can be proved using the methodology of \cite*{mei2016landscape}; we give a self-contained proof in Appendix~\ref{app:unif_conv} for completeness.

\begin{lemma}[uniform deviation guarantees]\label{lm:unif_conv} Under Assumptions~{\bf (A.1)} and {\bf(A.2)}, there exists an absolute constant $c_0$, such that the following holds, for $c = c_0 \left(1 \vee \log((M\vee L\vee (B+MR))R\sigma/\delta)\right)$ and $\sigma  \deq (A+(B+MR)R) \vee (B + MR) \vee (C + LR)$:
\begin{itemize}
	\item If $n \ge cd\log d$, then, with probability at least $1-\delta$,
	\begin{align}\label{eq:risk_concentration}
	\sup_{w \in \sB^d(R)} |F_\bZ(w)-F(w)| \le \sigma \sqrt{\frac{cd\log n}{n}}.
	\end{align}
	\item If $n \ge cd\log d$, then, with probability at least $1-\delta$,
	\begin{align}\label{eq:uniform_grad_conv}
		\sup_{w \in \sB^d(R)}\| \nabla F_\bZ(w) - \nabla F(w)\| \le \sigma \sqrt{\frac{cd\log n}{n}}.
	\end{align}
	\item If $n \ge cd\log d$, then, with probability at least $1-\delta$,
	\begin{align}
	\sup_{w \in \sB^d(R)} \| \nabla^2 F_\bZ(w) - \nabla^2 F(w)\| \le \sigma \sqrt{\frac{cd \log n}{n}}.
	\end{align}
\end{itemize}
\end{lemma}

\subsection{Empirical risk is strongly Morse with high probability}

The following result is due to \cite*{mei2016landscape}. We give a short proof, to keep the presentation self-contained.

\begin{proposition}\label{prop:strongly_Morse} If the population risk $F(w)$ is $(2\eps_0,2m)$-strongly Morse, then, provided $n$ satisfies the conditions of Lemma~\ref{lm:unif_conv} as well as the condition $\dfrac{n}{d\log n} \ge \dfrac{c\sigma^2}{(\eps_0 \wedge m)^2}$, the empirical risk $F_\bZ(w)$ is $(\eps_0,m)$-strongly Morse with probability at least $1-\delta$.
\end{proposition}

\begin{proof} 
By Lemma~\ref{lm:unif_conv}, the following event occurs with probability at least $1-\delta$:
	\begin{align*}
		\sup_{w \in \sB^d(R)} \|\nabla F_\bZ(w)-\nabla F(w)\| \le \eps_0 \qquad \text{and} \qquad
	 \sup_{w \in \sB^d(R)} \| \nabla^2 F_\bZ(w) - \nabla^2 F(w)\| \le m.
	\end{align*}
	On this event, for any $w \in \sB^d(R)$ satisfying $\min_{j \le d}|\lambda_j(\nabla^2 F(w))| \ge 2m$, 
	\begin{align*}
		\min_{j \le d}|\lambda_j(\nabla^2 F_\bZ(w))| &\ge \min_{j \le d}|\lambda_j \nabla^2 F(w)| - \max_{j \le d}|\lambda_j(\nabla^2 F_\bZ(w))-\lambda_j(\nabla^2 F(w))| \nonumber\\
		&\ge \min_{j \le d}|\lambda_j \nabla^2 F(w)|- \| \nabla^2 F_\bZ(w) - \nabla^2 F(w)\| \nonumber\\
		&\ge m,
	\end{align*}
by Weyl's perturbation theorem \citep[Corollary~III.2.6]{Bhatia_matrix_analysis}. Therefore, for any $w \in \sB^d(R)$ satisfying $\|\nabla F_\bZ(w)\| \le \eps_0$, we have $\|\nabla F(w)\| \le 2\eps_0$, in which case the absolute values of all the eigenvalues of $\nabla^2 F_\bZ(w)$ are at least $m$. Hence, the empirical risk $F_\bZ$ is $(\eps_0,m)$-strongly Morse.
\end{proof}

\subsection{Completing the proof: a posteriori risk bound}

Now let $\barw_\bZ$ be a local minimum of the empirical risk $F_\bZ$. By Proposition~\ref{prop:strongly_Morse}, all eigenvalues of the Hessian $H = \nabla^2 F_\bZ(\barw_\bZ)$ are at least $m$; therefore, the norm $\|\cdot\|_H$ is well defined, and satisfies $\|\cdot\|_H \ge \sqrt{m}\|\cdot\|$. We begin by decomposing
\begin{align*}
	F(\barw_\bZ) - \min_{k \in [K]} F_{\bZ}(W^{(k)}) &= \Big( F(\barw_\bZ) - F_{\bZ}(\barw_\bZ) \Big) + \Big( F_{\bZ}(\barw_\bZ) - \min_{k \in [K]} F_{\bZ}(W^{(k)}) \Big) \\
	&=: E_1 + E_2,
\end{align*}
where $E_1 \le \sigma\sqrt{(cd/n)\log n}$ with probability at least $1-\delta$, by Lemma~\ref{lm:unif_conv}. Next we control the term $E_2$. Under Assumption~{\bf(A.2)}, for any $w \in \Reals^d$,
	\begin{align*}
		\left|F_{\bZ}(w) - F_{\bZ}(\barw_\bZ)- \frac{1}{2} \| w - \barw_{\bZ} \|_H^2 \right| &\leq  \frac{L}{6} \| w - \barw_{\bZ} \|^3.
	\end{align*} 
\citep[Lemma~1.2.4]{nesterov2013introductory}. Therefore,
\begin{align*}
	E_2 &= \max_{K_1 \le k \le K} \left(F_\bZ(\barw_\bZ)-F_\bZ(W^{(k)})\right) \\
	&\le \max_{K_1 \le k \le K} \left(  \frac{L}{6m^{3/2}} \|W^{(k)}-\barw_\bZ\|^3_H  - \frac{1}{2}\|W^{(k)} - \barw_{\bZ} \|_H^2 \right).
\end{align*}
Then, by Theorem~\ref{thm:emp_metastable}, with probability $1-\delta$, either $\|W^{(k)}-\barw_\bZ\|_H \ge 2\eps$ for some $k \le K_0$ or $\|W^{(k)}-\barw_\bZ\|_H \le 2\eps$ for all $K_0 \le k \le K$. If the latter occurs, then $E_2 \le 0$ for $\eps \le \dfrac{3m^{3/2}}{2L}$.

\newpage

\acks{The authors would like to thank Matus Telgarsky for many enlightening discussions. The work of Belinda Tzen and Maxim Raginsky is supported in part by the NSF under CAREER award
CCF--1254041, and in part by the Center for Science of Information (CSoI), an NSF Science and
Technology Center, under grant agreement CCF--0939370. Tengyuan Liang is supported by George C.~Tiao Faculty Fellowship in Data Science Research.}

\bibliography{metastability_COLT.bbl}

\newpage

\appendix

   
   \section{Proof of Lemma~\ref{lm:mart_tail}}
   \label{app:mart_tail}
   
   \setcounter{theorem}{0}
   \renewcommand{\thetheorem}{\Alph{section}.\arabic{theorem}}
   
Since $\{Q_{t_0}(t_1)Z^0_t\}_{t \in [t_0,t_1]}$ is a martingale and the function $x \mapsto e^{\gamma x^2}$ is convex for $\gamma > 0$, $\{\exp[(\beta\lambda/2)\|Q_{t_0}(t_1)Z^0_t\|^2]\}_{t \in [t_0,t_1]}$ is a positive submartingale. Therefore, using Doob's maximal inequality,
    	\begin{align}
    		\PP^{y_0}\left[\sup_{t_0 \le t \le t_1} \|Q_{t_0}(t_1)Z^0_t\| \ge h\right] &\le \PP^{y_0} \left[\sup_{t_0 \le t \le t_1} e^{(\beta\lambda/2)\|Q_{t_0}(t_1)Z^0_t\|^2} \ge e^{\beta\lambda h^2/2}\right] \nonumber\\
    		&\le e^{-\beta\lambda h^2/2} \E^{y_0} \left[e^{(\beta\lambda/2)\|Q_{t_0}(t_1)Z^0_{t_1}\|^2}\right].\label{eq:Z0_Chernoff}
    	\end{align}
Since $Q_{t_0}(t_1)Z_{t_1}$ is a $d$-dimensional Gaussian random vector with mean $\mu$ and covariance matrix $\Sigma$, Lemma~\ref{lm:Gaussian_integral} below gives
    \begin{align}
    	\E^{y_0} \left[e^{(\beta\lambda/2)\|Q_{t_0}(t_1)Z^0_{t_1}\|^2}\right] 
    	&= \frac{1}{\sqrt{\det(I-\beta\lambda\Sigma)}} \exp\left(\frac{\beta\lambda}{2} \ave{\mu,(I-\beta\lambda\Sigma)^{-1}\mu}\right) \nonumber\\
		&\le \left(\frac{1}{1-\lambda}\right)^{d/2}\exp\left(\frac{\beta\lambda}{2} \ave{\mu,(I-\beta\lambda\Sigma)^{-1}\mu}\right), \label{eq:Z0_mgf}
    \end{align}
where the inequality follows from the fact that, by the choice of $\lambda$, the eigenvalues of $I - \beta\lambda \Sigma = (1-\lambda)I + \lambda e^{-2tH}$ are lower-bounded by $1-\lambda$. Putting \eqref{eq:Z0_Chernoff}--\eqref{eq:Z0_mgf} together, we get \eqref{eq:Z0_tail_bound}.

    \begin{lemma}\label{lm:Gaussian_integral}
    Consider a random variable $V \sim \cN(\mu,\Sigma)$. Then, for any $\gamma > 0$ such that $I - 2\gamma\Sigma$ is positive definite,
    \begin{align}\label{eq:Gaussian_integral}
    	\E[e^{\gamma\|V\|^2}] = \frac{1}{\sqrt{\det (I - 2\gamma\Sigma)}} \exp\left(\gamma \ave{\mu,(I-2\gamma\Sigma)^{-1}\mu}\right).
    \end{align}
    \end{lemma}

    \begin{proof} We have
    	\begin{align}
    		\E[e^{{\gamma}\|V\|^2}] &= \frac{1}{\sqrt{\det(2\pi \Sigma)}}\int_{\Reals^d} e^{-\frac{1}{2}\ave{v-\mu,\Sigma^{-1}(v-\mu)}}e^{{\gamma} \ave{v,v}} \d v \nonumber\\
    		&= \frac{1}{\sqrt{\det(2\pi \Sigma)}} \int_{\Reals^d}e^{-\frac{1}{2}[\ave{v-\mu,\Sigma^{-1}(v-\mu)}-2{\gamma}\ave{v,v}]}\d v.\label{eq:Gauss_integral}
    	\end{align}
We proceed to complete the square in the exponent:
    	\begin{align}
   	&	\ave{v-\mu,\Sigma^{-1}(v-\mu)}-2{\gamma}\ave{v,v}  \nonumber\\
    	& \qquad = \ave{v,\underbrace{(\Sigma^{-1}-2{\gamma}I)}_{S}v} - 2\ave{\underbrace{\Sigma^{-1}\mu}_{Su}, v} + \ave{\mu,\Sigma^{-1}\mu} \nonumber \\
    	&= \ave{v-u,S(v-u)}+ \ave{(I-S^{-1}\Sigma^{-1})\mu,\Sigma^{-1}\mu} \nonumber \\
    	&= \ave{v-u,S(v-u)} + \ave{\mu,(\Sigma^{-1}-\Sigma^{-1}S^{-1}\Sigma^{-1})\mu}.\label{eq:complete_squares}
    	\end{align}
Now we recall the Woodbury matrix identity: If $B = A + UV$, such that $A$ and $I + VA^{-1}U$ are invertible, then
    \begin{align}\label{eq:Woodbury}
    	B^{-1} = A^{-1} - A^{-1}U(I + VA^{-1}U)^{-1}VA^{-1}.
    \end{align}
Applying \eqref{eq:Woodbury} with $B = S$, $A = \Sigma^{-1}$, $U = -2{\gamma}I$, $V = I$, we can write
    \begin{align*}
    	S^{-1} &= (\Sigma^{-1}-2{\gamma}I)^{-1} \\
    	&= \Sigma + 2{\gamma}\Sigma(I - 2{\gamma}\Sigma)^{-1}\Sigma,
    \end{align*}
    which gives
    \begin{align*}
    	\Sigma^{-1}-\Sigma^{-1}S^{-1}\Sigma^{-1} &= \Sigma^{-1} - \Sigma^{-1} \left(\Sigma + 2{\gamma}\Sigma(I - 2{\gamma}\Sigma)^{-1}\Sigma \right)\Sigma^{-1} \\
    	&= \Sigma^{-1} - \Sigma^{-1} - 2{\gamma}(I-2{\gamma}\Sigma)^{-1} \\
    	&= - 2{\gamma}(I-2{\gamma}\Sigma)^{-1} .
    \end{align*}
    Substituting this into \eqref{eq:complete_squares}, we get
    \begin{align*}
    	\ave{v-\mu,\Sigma^{-1}(v-\mu)}-2{\gamma}\ave{v,v} &= \ave{v-u,S(v-u)} - 2{\gamma}\ave{\mu,(I-2{\gamma}\Sigma)^{-1}\mu}.
    \end{align*}
    We can now compute the Gaussian integral in \eqref{eq:Gauss_integral}:
    \begin{align*}
    	&\frac{1}{\sqrt{\det(2\pi \Sigma)}} \int_{\Reals^d}e^{-\frac{1}{2}[\ave{v-\mu,\Sigma^{-1}(v-\mu)}-2{\gamma}\ave{v,v}]}\d v \\
    	&= \frac{1}{\sqrt{\det (S\Sigma)}} \exp\left({\gamma}\ave{\mu,(I-2{\gamma}\Sigma)^{-1}\mu}\right) \\
    	&=\frac{1}{\sqrt{\det (I-2{\gamma}\Sigma)}} \exp\left({\gamma}\ave{\mu,(I-2{\gamma}\Sigma)^{-1}\mu}\right),
    \end{align*}
	which gives \eqref{eq:Gaussian_integral}.
    \end{proof}

   \section{Proof of Lemma~\ref{lm:unif_conv}}
   \label{app:unif_conv}
	
	\paragraph{Subgaussian noise conditions.} For any $w \in \sB^d(R)$, we have
   \begin{align*}
	   |f(w,Z)| &\le |f(0,Z)| + |f(w,Z)-f(0,Z)| \\
	   &\le |f(0,Z)| + \sup_{v \in \sB^d(R)}\|\nabla f(v,Z)\| \|w\| \\
	   &\le |f(0,Z)| + \left(\|\nabla f(0,Z)\| + MR\right)\|w\| \\
	   &\le A + (B+MR)R.
	\end{align*}
	Thus, by the Hoeffding lemma, the random variable $f(w,Z)$ is subgaussian: for any $t > 0$
		\begin{align}\label{eq:0th_order_noise}
			\E \exp\left( t \Big(f(w, Z) - F(w)\Big)\right) \leq \exp\left( \frac{\sigma_0^2 t^2}{2} \right)
		\end{align}
	with $\sigma^2_0 = [A+(B+MR)R]^2$. For any $w \in \sB^d(R)$ and $v \in \sB^d(1)$, we have
	\begin{align*}
		|\langle{v, \nabla f(w,Z)\rangle}| &\le \|v\| \|\nabla f(0,Z)\| + \|v\| \|\nabla f(w,Z)-\nabla f(0,Z)\| \\
		&\le B + MR.
	\end{align*}
	Thus, by the Hoeffding lemma, the random variable $\ave{v,\nabla f(w,Z)}$ is subgaussian: for any $t > 0$,
		\begin{align}\label{eq:1st_order_noise}
			\E \exp\left( t \langle v, \nabla f(w, Z) - \nabla F(w) \rangle  \right) \leq \exp\left( \frac{\sigma_1^2 t^2}{2} \right)
		\end{align}
with $\sigma_1^2 = (B + MR)^2$. In the same vein,
	\begin{align*}
		|\ave{v, \nabla^2 F(w,Z)v}| &\le \|v\|^2 \|\nabla^2 F(0,Z)\| + \|v\|^2 \|\nabla^2 F(w,Z)-\nabla^2 F(0,Z)\| \\
		&\le C + LR,
	\end{align*}
hence, for any $t > 0$,
	\begin{align}\label{eq:2nd_order_noise}
	\E \exp\left( t \left\langle v,\big(\nabla^2 f(w, Z) - \nabla^2 F(w)\big)v \right\rangle  \right) \leq \exp\left( \frac{\sigma_2^2 t^2}{2} \right)
	\end{align}
with $ \sigma_2^2 = (C + LR)^2$.

\paragraph{Uniform deviation bound for the risk.} For any two $w,w' \in \sB^d(R)$, 
\begin{align*}
	|f(w,Z)-f(w',Z)| &\le \sup_{v \in \sB^d(R)} \|\nabla f(v,Z)\| \|w-w'\| \\
	&\le (B+MR)\|w-w'\|.
\end{align*}
Therefore, if we let $\{w_1,\ldots,w_N\}$ be an $\eps$-cover of $\sB^d(R)$, with $N \le (3R/\eps)^d$, then
\begin{align*}
	\sup_{w \in \sB^d(R)}\left| \frac{1}{n}\sum^n_{i=1} f(w,Z_i) -  F(w)\right| \le \max_{1 \le j \le N} \left|\frac{1}{n}\sum^n_{i=1} f(w_j,Z_i) -  F(w_j)\right| + 2(B+MR)\eps.
\end{align*}
Hence, for any $t > 4(B+MR)\eps$, we have
\begin{align*}
&	\PP\left[ \sup_{w \in \sB^d(R)}\left| \frac{1}{n}\sum^n_{i=1} f(w,Z_i) -  F(w)\right| \ge t \right] \\
&\le \PP\left[ \max_{1 \le j \le N}\left| \frac{1}{n}\sum^n_{i=1} f(w_j,Z_i) - F(w_j)\right| \ge \frac{t}{2} \right]  \\
&\le \left(\frac{3R}{\eps}\right)^d \max_{1 \le j \le N}\PP\left[\left|\frac{1}{n}\sum^n_{i=1}f(w_j,Z_i)-F(w_j)\right| \ge \frac{t}{2}\right] \\
&\le \left(\frac{6R}{\eps}\right)^d \exp\left(-\frac{nt^2}{8\sigma^2_0}\right),
\end{align*}
where we have used the fact that $f(w,Z)$ is $\sigma^2_0$-subgaussian [cf.~\eqref{eq:0th_order_noise}].
Choosing $\eps = \frac{\sigma_0}{4(B+MR)dn}$ and
\begin{align*}
	t \ge \frac{\sigma_0}{dn} \vee \sqrt{\frac{8\sigma^2_0}{n}\left(d \log \frac{24(B+MR)Rdn}{\sigma_0}+\log\frac{1}{\delta}\right)},
\end{align*}
we see that \eqref{eq:risk_concentration} holds with probability at least $1-\delta$ for a suitable choice of $c_0$.

\paragraph{Uniform deviation bound for the gradient.}
Let $\{v_1,\ldots,v_J\}$ be a $(1/2)$-cover of $\sB^d(1)$, with $J \le 6^d$. Then, for any $w \in \sB^d(R)$, using the fact that $\ave{v,f(w,Z)}$ is $\sigma^2_1$-subgaussian [cf.~\eqref{eq:1st_order_noise}],
\begin{align*}
&	\PP\left[\left\| \frac{1}{n}\sum^n_{i=1}\nabla f(w,Z_i) - \nabla F(w)\right\| \ge t \right] \\
& \le \PP\left[ \max_{1 \le j \le J} \left\langle v_j, \frac{1}{n}\sum^n_{i=1}\nabla f(w,Z_i)-\nabla F(w)\right\rangle \ge \frac{t}{2}\right] \\
& \le 6^d \exp\left(-\frac{nt^2}{8\sigma^2_1}\right).
\end{align*}
Next, let $\{w_1,\ldots,w_N\}$ be an $\eps$-cover of $\sB^d(R)$, with $N \le (3R/\eps)^d$. Then, since
\begin{align*}
	\sup_{w \in \sB^d(R)}\left\| \frac{1}{n}\sum^n_{i=1}\nabla f(w,Z_i) - \nabla F(w)\right\| \le \max_{1 \le j \le N} \left\|\frac{1}{n}\sum^n_{i=1}\nabla f(w_j,Z_i) - \nabla F(w_j)\right\| + 2M\eps,
\end{align*}
for any $t > 4M\eps$ we have
\begin{align*}
&	\PP\left[ \sup_{w \in \sB^d(R)}\left\| \frac{1}{n}\sum^n_{i=1}\nabla f(w,Z_i) - \nabla F(w)\right\| \ge t \right] \\
&\le \PP\left[ \max_{1 \le j \le N}\left\| \frac{1}{n}\sum^n_{i=1}\nabla f(w_j,Z_i) - \nabla F(w_j)\right\| \ge \frac{t}{2} \right]  \\
&\le \left(\frac{3R}{\eps}\right)^d \max_{1 \le j \le N} \PP\left[\left\| \frac{1}{n}\sum^n_{i=1}\nabla f(w_j,Z_i) - \nabla F(w_j)\right\| \ge \frac{t}{2} \right] \\
&\le \left(\frac{18R}{\eps}\right)^d \exp\left(-\frac{nt^2}{32\sigma^2_1}\right).
\end{align*}
Choosing $\eps = \frac{\sigma_1}{4Mdn}$ and
\begin{align*}
	t \ge \frac{\sigma_1}{dn} \vee \sqrt{\frac{32\sigma^2_1}{n}\left(d \log \frac{72MRdn}{\sigma_1}+\log\frac{1}{\delta}\right)},
\end{align*}
we see that \eqref{eq:uniform_grad_conv} holds with probability at least $1-\delta$ for a suitable choice of $c_0$.

\paragraph{Uniform deviation bound for the Hessian.} Now, let $\{v_1,\ldots,v_J\}$ be an $(1/4)$-cover of $\sB^d(1)$, with $J \le 12^d$. Then, for any $w \in \sB^d(R)$, using the fact that $\ave{v,\nabla^2 f(w,Z)v}$ is $\sigma^2_2$-subgaussian [cf.~\eqref{eq:2nd_order_noise}],
\begin{align*}
&	\PP\left[\left\| \frac{1}{n}\sum^n_{i=1}\nabla^2 f(w,Z_i) - \nabla^2 F(w)\right\| \ge t \right] \\
& \le \PP\left[ \max_{1 \le j \le J} \left\langle v_j, \left(\frac{1}{n}\sum^n_{i=1}\nabla^2 f(w,Z_i)-\nabla^2 F(w)\right)v_j\right\rangle \ge \frac{t}{2}\right] \\
& \le 12^d \exp\left(-\frac{nt^2}{8\sigma^2_2}\right).
\end{align*}
Again, let $\{w_1,\ldots,w_N\}$ be an $\eps$-cover of $\sB^d(R)$. Then, since
\begin{align*}
	\sup_{w \in \sB^d(R)}\left\| \frac{1}{n}\sum^n_{i=1}\nabla^2 f(w,Z_i) - \nabla^2 F(w)\right\| \le \max_{1 \le j \le N} \left\|\frac{1}{n}\sum^n_{i=1}\nabla^2 f(w_j,Z_i) - \nabla^2 F(w_j)\right\| + 2L\eps,
\end{align*}
for any $t > 4L\eps$ we have
\begin{align*}
&	\PP\left[ \sup_{w \in \sB^d(R)}\left\| \frac{1}{n}\sum^n_{i=1}\nabla^2 f(w,Z_i) - \nabla^2 F(w)\right\| \ge t \right] \\
&\le \PP\left[ \max_{1 \le j \le N}\left\| \frac{1}{n}\sum^n_{i=1}\nabla^2 f(w_j,Z_i) - \nabla^2 F(w_j)\right\| \ge \frac{t}{2} \right]  \\
&\le \left(\frac{3R}{\eps}\right)^d \max_{1 \le j \le N} \PP\left[\left\| \frac{1}{n}\sum^n_{i=1}\nabla^2 f(w_j,Z_i) - \nabla^2 F(w_j)\right\| \ge \frac{t}{2} \right] \\
&\le \left(\frac{36R}{\eps}\right)^d \exp\left(-\frac{nt^2}{32\sigma^2_1}\right).
\end{align*}
Choosing $\eps = \frac{\sigma_2}{4Ldn}$ and
\begin{align*}
	t \ge \frac{\sigma_2}{dn} \vee \sqrt{\frac{32\sigma^2_2}{n}\left(d \log \frac{144LRdn}{\sigma_2}+\log\frac{1}{\delta}\right)},
\end{align*}
we see that \eqref{eq:uniform_grad_conv} holds with probability at least $1-\delta$ for a suitable choice of $c_0$.

\end{document}